\documentclass[15pt]{article}
\usepackage{amsmath,amssymb,amsthm}
\usepackage{graphicx}
\usepackage[left=30mm,right=30mm,top=30mm,bottom=30mm,a4paper]{geometry}
\usepackage{fancyhdr}

\fancypagestyle{firstStyle}
{
   \fancyhf{}
   
   \fancyhead[C]{International Journal of Artificial Intelligence and Applications (IJAIA), Vol.8, No.2, March 2017}
   \fancyfoot[R]{\thepage}
   \fancyfoot[L]{DOI : 10.5121/ijaia.2017.8205}
}

\fancypagestyle{mainStyle}
{
   \fancyhf{}
   
   \fancyhead[C]{International Journal of Artificial Intelligence and Applications (IJAIA), Vol.8, No.2, March 2017}
   \fancyfoot[R]{\thepage}
}

\title{\textbf{Source-Sensitive Belief Change}}
\author{Shahab Ebrahimi\\
Amirkabir University of Technology, Tehran, Iran\\
sh.ebrahimi@aut.ac.ir}
\date{}

\newtheorem{theorem}{Theorem}[section]
\newtheorem{definition}{Definition}[section]

\begin{document}

\maketitle

\textbf{\textit{ABSTRACT}}

\textit{The AGM model is the most remarkable framework for modeling belief revision. However, it is not perfect in all aspects. Paraconsistent belief revision, multi-agent belief revision and non-prioritized belief revision are three different extensions to AGM to address three important criticisms applied to it. In this article, we propose a framework based on AGM that takes a position in each of these categories. Also, we discuss some features of our framework and study the satisfiability of AGM postulates in this new context.}


\section{Introduction}

\thispagestyle{firstStyle}

Belief revision is a scientific field of research in the intersection of epistemology, logic and artificial science. The main goal of belief revision is to provide a logical framework for modeling the process of belief change of a rational agent. AGM (\cite{AGM85, Mak85, Gär88}) is the most popular model for this aim. In AGM, a knowledge state is represented by a logically closed set of propositions called belief set. Given a proposition as an input, three types of change are possible: expansion, revision and contraction. Expansion and revision are both about adding a new belief to the belief set, but in the former, keeping the consistency of the set is a consideration. Contraction, however, deals with retracting an old belief. Applying any of these changes to a belief set results a new set that according to AGM should satisfy a set of postulates ordained with respect to several rational criteria.

Beside all advantages of AGM that have been explored by authors, there are several criticisms have been noticed to show that it has some counterintuitive features. Thus, various extensions to AGM are proposed to address such criticisms (\cite{FH11}). Multi-agent belief revision, non-prioritized belief revision and paraconsistent belief revision are three extensions proposed for three different criticisms of AGM.

Multi-agent belief revision is introduced in order to enable AGM to handle belief changes in multi-agent environments; whether we face multiple agents that change their beliefs, or have multiple sources that utters new propositions for making belief changes. The latter case is called multi-source belief revision.

The main goal of non-prioritized belief revision is different. In AGM, every input information has higher priority than existing beliefs. In other words, accepting a new proposition or retracting a non-tautological one could be done without any obligation. However, this feature of AGM seems to be an ideal assumption since we intuitively don't accept every information passed to us by every source. All multi-source models could be considered as non-prioritized, because it is impossible to differentiate between the sources if we accept all the inputs of all of them. However, the reverse is not true. There are several non-prioritized approaches that are defined in single-agent environments and are sensitive only to input propositions.

In paraconsistent belief revision, the main problem is that it is impossible to have two different inconsistent belief sets. Because by the rules of classical logic, every inconsistent belief set is equivalent to $K_\bot$, the set of all propositions. This feature seems counterintuitive too. Whether the propositions are perceived as rational beliefs or practical data, it seems plausible to be able to make meaningful deductions in case of inconsistencies. Also, handling such situations seems reasonable in multi-source context. Because even fully self-consistent and rational sources can utter contradicting information and in some cases, there are not enough reasons to accept only one of them.

As it is clear, these three different criticisms have some intersections in their motivations and it is reasonable to study models that coalesce all of them together. In the following, we try to define a framework based on AGM that leverages the benefits of the three aforementioned extensions.

\section{The Approach}

In this section, we describe our approach for handling the mentioned criticisms.

\subsection{For Paraconsistency}

First change we apply, is to use a paraconsistent logic as our background logic. Paraconsistent logic is a class of non-classical logics for which ECQ\footnote{ex contradiction quodlibet} fails, i.e., the relation $q \in Cn(\{p, \neg p \})$ is not valid in those systems. Therefore, they allow a theory to contain inconsistencies while not deducing all propositions. There are many approaches for constructing such systems (\cite{Pri02}) and the question is which one performs the best in the AGM model. Here, we use a three-valued paraconsistent logic, called PAC (\cite{Avr91}). This logic is an enrichment of the logic LP (\cite{Pri79}) with an implication connective $\to$ for which modus ponens\footnote{$\{p , p \to q \} \vdash q$.} and the deduction theorem\footnote{If $A \cup \{ p \} \vdash q$, then $A \vdash p \to q$.} hold for it. In PAC, we have three values, denoted by $1$, $-1$ and $0$. The interpretations of them are respectively ``true'', ``false'' and ``both true and false''. Hence, the first two values are like the values of classical logic for representing ``truth'' and ``falsity'' and the third value, $0$, refers to ``paradoxical'' propositions, which are both true and false at the same time. We take $1$ and $0$ as designated values in PAC. Therefore, the set of values in PAC is $V=\{1,0,-1\}$ and the set of designated values is $D=\{1,0\}$. In PAC, four connectives $\land$, $\lor$, $\to$ and $\neg$ are defined respectively for ``conjunction (and)'', ``disjunction (or)'', ``material implication (if...then)'' and ``negation (not)'' by table \ref{PACCM}. Two constants $\bot$ and $\top$ are also given with the values ``false'' and ``true''.

\begin{table}[h!]
\centering
\renewcommand{\arraystretch}{1.5}
\begin{tabular} {|c|c c c|c|c|c c c|c|c|c c c|c|c|c|}
\cline{1-4}
\cline{6-9}
\cline{11-14}
\cline{16-17}
$\land$ &1&0&-1&&$\lor$&1&0&-1&&$\to$&1&0&-1&&$\neg$&\\
\cline{1-4}
\cline{6-9}
\cline{11-14}
\cline{16-17}
1&1&0&-1&&1&1&1&1&&1&1&0&-1&&1&-1\\
0&0&0&-1&&0&1&0&0&&0&1&0&-1&&0&0\\
-1&-1&-1&-1&&-1&1&0&-1&&-1&1&1&1&&-1&1\\
\cline{1-4}
\cline{6-9}
\cline{11-14}
\cline{16-17}
\end{tabular}
\caption{The matrices of the connectives in PAC.}
\label{PACCM}
\end{table}

The consequence relation of PAC ($\vdash$) is defined as follows: $\{p_1,p_2,...,p_n \} \vdash p$ if and only if every valuation which assigns either $1$ or $0$ to all $p_i$, does the same to $p$. It is possible for a valuation $v$, such that $v(p)=v(\neg p)=0$ and $v(q)=-1$. Therefore, $\{ p , \neg p \} \nvdash q$ and ECQ fails in this logic. The Hilbert-style formulation of PAC can be given by table \ref{NSPAC}. In this table, the connective $\leftrightarrow$ is defined as follows: $p \leftrightarrow q:= (p \to q) \land (q \to p)$.

\begin{table}[h!]
\centering
\renewcommand{\arraystretch}{1.5}
\begin{tabular}{|l|}
\cline{1-1}
Axioms:\\
\qquad \textbullet \quad $p \to (q \to p)$\\
\qquad \textbullet \quad $(p \to (q \to r)) \to ((p \to q) \to (p \to r))$\\
\qquad \textbullet \quad $((p \to q) \to p) \to p$\\
\qquad \textbullet \quad $p \land q \to p$\\
\qquad \textbullet \quad $p \land q \to q$\\
\qquad \textbullet \quad $p \to (q \to p \land q)$\\
\qquad \textbullet \quad $p \to p \lor q$\\
\qquad \textbullet \quad $q \to p \lor q$\\
\qquad \textbullet \quad $(p \to r) \to ((q \to r) \to (p \lor q \to r))$\\
\qquad \textbullet \quad $\neg (p \lor q) \leftrightarrow \neg p \land \neg q$\\
\qquad \textbullet \quad $\neg (p \land q) \leftrightarrow \neg p \lor \neg q$\\
\qquad \textbullet \quad $\neg \neg p \leftrightarrow p$\\
\qquad \textbullet \quad $\neg (p \to q) \leftrightarrow p \land \neg q$\\
\qquad \textbullet \quad $p \lor \neg p$\\
\cline{1-1}
Rule of Inference:\\
\qquad \textbullet \quad
\begin{tabular}{c}
$p \qquad p \to q$\\
\hline
$q$\\
\end{tabular}\\
\cline{1-1}
\end{tabular}
\caption{Hilbert-style system of PAC.}
\label{NSPAC}
\end{table}

PAC has lots of advantages that makes it a suitable logic for being used for our purposes. On the one hand, by the ideas of society semantics and multi-source epistemic logics introduced in \cite{CL99}, \cite{CD12}, \cite{CD13} and \cite{DP14}, it can be shown that PAC is a reasonable logic for handling inconsistencies in multi-source environment with respect to epistemological aspects. On the other hand, the profits of paraconsistent three-valued logics in handling inconsistent data bases, is discussed in \cite{CMD04}. Thus, it seems PAC is a suitable candidate for our framework both conceptually and practically. Furthermore, it is an ideal and natural paraconsistent logic, as discussed in \cite{AAZ11a}, \cite{AAZ11b} and \cite{AA15} that confirms it behaves like classical logic in the majority of cases.

By having the consequence relation of PAC, we can define its induced logical consequences operator as follow:

\begin{equation}
Cn(A)=\{p: A' \vdash p, \textit{ for a finite } A' \subseteq A \}
\end{equation}

Consequently, the theorem below could be concluded:

\begin{theorem}
$Cn$ satisfies the following conditions:
\begin{itemize}
\item [($Cn$1)]
$A \subseteq Cn(A)$. (inclusion)

\item [($Cn$2)]
$Cn(Cn(A)) \subseteq Cn(A)$. (iteration)

\item [($Cn$3)]
If $A \subseteq B$, then $Cn(A) \subseteq Cn(B)$. (monotony)

\item [($Cn$4)]
If $p \in Cn(A)$, then $p \in Cn(A') $ for some finite set $A'$ of $A$. (compactness)

\item [($Cn$5)]
If $p \in Cn(A \cup \{ q \})$ and $p \in Cn(A \cup \{ r \})$, then $p \in Cn(A \cup \{ q \lor r \})$. (introduction of disjunction in premises)
\end{itemize}
\end{theorem}

\begin{proof}
By the definition of $\vdash$, it is easy to see that it satisfies Tarskian conditions on consequence relations, i.e., reflexivity\footnote{If $p \in A$, then $A \vdash p$.}, transitivity\footnote{If $A \vdash p$ and $B \vdash q$ for every $q \in A$, then $B \vdash p$.} and weakening\footnote{If $A \vdash p$ and $A \subseteq B$, then $B \vdash p$.}. ($Cn$1)-($Cn$3) are the direct results of these conditions.

The satisfiability of ($Cn$4) can be easily followed by definition. Because if $p \in Cn(A)$, it must be at least one finite $A' \subseteq A$, such that $A' \vdash p$. And finally, by the semantics of PAC for $\lor$ and the definition of $\vdash$, ($Cn$5) can be easily deduced.
\end{proof}

\subsection{For Multi-sourceness}

Like the original case, by a belief set $K$ we mean a set of propositions that is closed under logical consequences, i.e., $K=Cn(K)$. However, unlike AGM, the inputs are not represented by a single proposition. In the proposed approach, an input $I$ is described by an ordered pair $(p,s)$ such that the first element is a proposition and the second term is the source that utters $p$ for a particular type of belief change. In the following, we will use $P_I$ and $S_I$ for referring to the proposition and the source of input $I$.

Also, epistemic and reliability functions are utilized to specify the degree of believability for beliefs and inputs, respectively. In order to define these functions, a value set for the range of these functions is required. Each value set $V$ needs at least three elements with a given total order on them, i.e., for every $x,y,z \in V$ we have:

\begin{itemize}
\item
$x \leq y$ or $y \leq x$. (totality)

\item
If $x \leq y$ and $y \leq x$, then $x=y$. (antisymmetry)

\item
If $x \leq y$ and $y \leq z$, then $x \leq z$. (transitivity)
\end{itemize}

It is also assumed that with respect to $\leq$, $V$ has minimum element $b$ and maximum element $t$. In the recent article, $V$ is taken as a fix value set. By adopting the notion of epistemic entrenchment proposed in \cite{GM88}, the notion of epistemic function can be defined. Motivations and properties of epistemic entrenchment fit well to our approach. There is only one critical difference in our framework. Epistemic entrenchment is an order on propositions providing a qualitive relation between them, but here the epistemic function is considered to be a function on propositions and the range of it gives such a relation on propositions. If $PROP$ is the set of all propositions of our language, then we have:

\begin{definition}
For a given belief set $K$, the function $E:PROP \to V$ is an epistemic function if it satisfies:

\begin{description}
\item [($E$1)]
If $p \vdash q$ the $E(p) \leq E(q)$. (dominance)

\item [($E$2)]
For any $p$ and $q$, $E(p) \leq E(p \land q)$ or $E(q) \leq E(p \land q)$. (conjunctiveness)

\item [($E$3)]
When $K \neq K_\bot$, $p \notin K$ if and only if $E(p)=b$. (minimality)

\item [($E$4)]
If $E(p)=t$, then $\vdash p$. (maximality)
\end{description}
\end{definition}

The relation $E(q) \leq E(p)$ states that $p$ is epistemologically more important than $q$, in the sense of retaining to $K$ if one of them must be avoided. The motivations of considering these conditions can be found in \cite{GM88}. Beside epistemic function, we assume another function on inputs that shows the reliabilities of inputs. The notions of reliability and trustworthiness are usually taken as basic factors in communication and social sciences and the effects of them in belief change have been discussed in many articles (\cite{WRK12, Liu02}). For the input $(p,s)$, the reliability of the source $s$ and the amount of its knowledge about $p$ are both important for evaluating the degree of reliability.

Most models for multi-source or non-prioritized belief revision relate such concepts only to input proposition or input source. However, many situations can be given to show that it is plausible to assume both are important for such evaluation. For example, the relation between the trustworthiness of two sources that are expert in two different science fields may not be the only important data, but the content of information they utter may be effective either. This evaluation plays the main role in several field of science and we don't challenge the difficulties of that. Here, it is assumed that such information about reliabilities are given, like epistemic entrenchment order. However, we do not claim that computing reliabilities is an easy task to do. Just devolve it to another research field. 

If $INP$ is the set of all possible inputs, then we have:

\begin{definition}
The function $R: INP \to V$ is a reliability function if it satisfies:
\begin{description}
\item[($R$)]
If $I_1=(p,s)$ and $I_2=(q,s)$ and $p \vdash q$, then $R(I_1) \leq R(I_2)$. (dominance)
\end{description}
\end{definition}

By ($R$), it is said that if a single source utters two propositions such that one of them is the consequence of the other one, then the reliability of the input with the deduced proposition must be greater than the other. Because if $p \vdash q$, then $Cn(q) \subseteq Cn(p)$. Hence, $p$ can be seen as a bigger claim than $q$ for making a change in the beliefs. Intuitively, if the sources are not the same, this postulate should not necessarily hold. Several other conditions can be considered for reliability function, but we will continue with this general definition to keep the situation open for any possible extension in future.

\subsection{For Non-priority}

The properties of $\leq$ on $V$ provide the prerequisites to do a comparison between old beliefs and inputs to make a decision for belief change. Most works for this purpose just get the non-priority of the revision functions under discussion. In \cite{Fer99} two taxonomies are given for classifying the non-prioritized functions proposed for this change; one based on the outputs of them and the other based on their construction process. In the recent framework, the non-priority of every kind of belief changes will be discussed.

Hence, we generalize our approach in order to suit source-sensitive functions, as follows:

\begin{itemize}
\item
Based on the output: ``all or nothing'' approach; that means either it accepts the whole input or leaves the belief set unchanged. 

\item
Based on the construction: ``decision + action'' approach; that means first we decide whether to accept the input or not, then do the change. 
\end{itemize}

\section{Source-sensitive Belief Changes}

In the original manuscripts about AGM, contraction, expansion and revision are considered to be the three main types of belief change that happen to a belief state with respect to to an input. Contraction is defined for eliminating an existing belief and expansion together with revision are defined for adding a new belief. The difference between revision and expansion is that in the former, the result must be consistent, whereas for expansion consistency in not a restriction. By changing our view from consistency to paraconsistency, motivations for taking revision as a main kind of belief change don't have reasonable supports and motivations anymore. As investigated by Tanaka (\cite{Tan05}), by using paraconsistent logics in the AGM model, revision collapses on expansion.

Also, it seems very admissible to agree with this result in our multi-source environment. Consistency of beliefs in this framework can be considered as an effective factor for evaluating reliability, however it is not the only important parameter. Hence, the comparison of $R(I)$ and $E(\neg P_I)$ does not provide sufficient information for accepting or rejecting $I$. Therefore, there cannot be a direct way to define revision and we take expansion and contraction as the only main kinds of belief change.

\subsection{Expansion}

By expanding a belief set $K$ by an input $I$, we want to non-prioritizely have $P_I$ as a belief in the resulted belief set. Since no restriction is considered in the definition of expansion, every new information is welcome. The only condition we consider is $b < R(I)$, that means the reliability of an acceptable input can be anything expect exactly the minimum degree of reliability, that is the amount of the epistemic value of $\bot$. So, we define expansion as follows:

\begin{definition}
For a belief set $K$ and an input $I$, the function $\dot{+}$ is a source-sensitive expansion if and only if:
\begin{center}
$K \dot{+} I=
\begin{cases}
Cn(K \cup \{ P_I \})&\textit{if } b<R(I)\\
K&\textit{otherwise}
\end{cases}$
\end{center}
\end{definition}

The construction is simple. If $b<R(I)$, we add $P_I$ to $K$ and close it under $Cn$, to accept the new logical consequences of $K \cup \{ p \}$ too. On contrary, unlike original cases, our background logic is PAC and also the change may be unsuccessful. If we want to discuss the properties of this function by studying the satisfiability of AGM postulates, we must introduce the modified version of them to have compatible postulates with our framework and its notations. The corresponding postulates to AGM expansion postulates are:

\begin{description}
\item [($\dot{+}$1)]
$K \dot{+} I$ is a belief set. (closure)

\item [($\dot{+}$2)]
$P_I \in K \dot{+} I$ or $K \dot{+} I=K$. (relative success)$^{\ast}$

\item [($\dot{+}$3)]
$K \subseteq K \dot{+} I$. (inclusion)

\item [($\dot{+}$4)]
If $P_I \in K$, then $K \dot{+} I=K$. (vacuity)

\item [($\dot{+}$5)]
If $K_1 \subseteq K_2$, then $K_1 \dot{+} I \subseteq K_2 \dot{+} I$. (monotony)

\item [($\dot{+}$6)]
$K \dot{+} I$ is smaller than any set that satisfies ($\dot{+}$1) and ($\dot{+}$3)-($\dot{+}$5) and contains $P_I$. (minimality)
\end{description}

The only postulate that is not the direct correspondence of its corresponding AGM postulate, is ($\dot{+}$2). Here, we use a weakening of the success postulate\footnote{Success postulate state that if $+$ is an AGM expansion, then $p \in K+p$.}, that shows ``all or nothing'' approach for constructing our non-prioritized expansion function. Also, in the original case, the last postulate specifies a limit on the size of the result. Thus, the important thing is to have a control on the size of the expanded belief set. Therefore, we consider postulate ($\dot{+}$6) only for the case that $P_I$ is accepted in the result. Now, for this function we can show:

\begin{theorem}
Source-sensitive expansion satisfies ($\dot{+}$1)-($\dot{+}$6).
\end{theorem}

\begin{proof}
For ($\dot{+}$1), if $b<R(I)$, then by definition, the result is a belief set. Otherwise, the result is $K$ that is our given belief set. For ($\dot{+}$2), if $b < R(I)$, then by ($Cn$1), $P_I \in Cn(K \cup \{ P_I \})$. If not $b < R(I)$, then by definition $K \dot{+} I=K$.

For ($\dot{+}$3), if $b<R(I)$, then by ($Cn$3), $K \subseteq Cn(K \cup \{ P_I \})$. In the other case, the result is $K$ and $K \subseteq K$. For ($\dot{+}$4), suppose that $b<R(I)$. From $P_I \in K$, it follows that $K \dot{+} I=Cn(K \cup \{p\})=Cn(K)=K$. If not $b <R(I)$, then the result is $K$ by definition. For ($\dot{+}$5), first assume $b<R(I)$. Since $K_1 \subseteq K_2$, by ($Cn$3) we conclude that $Cn(K_1 \cup \{P_I\}) \subseteq Cn(K_2 \cup \{P_I \})$. So, $K_1 \dot{+} I \subseteq K_2 \dot{+} I$. For the other case, by definition $K_1 \dot{+} I=K_1$ and $K_2 \dot{+} I=K_2$. Hence by assumption, $K_1 \dot{+} I \subseteq K_2 \dot{+} I$.

For ($\dot{+}$6), assume that $K^\ast$ satisfies ($\dot{+}$1) and ($\dot{+}$3)-($\dot{+}$5) and contains $P_I$. Assume the case that $b<R(I)$. Suppose that not $K \dot{+} I \subseteq K^\ast$. It means that, there is a proposition $p$ such that $p \in K \dot{+} I$ and $p \notin K^\ast$. Since $p \in K \dot{+} I$, then by definition, there is a finite $K' \subseteq K \cup \{ P_I \}$ such that $K' \vdash p$. We have two possible conditions. Either $P_I \notin K'$ or $P_I \in K'$. The former case deduces that $K' \subseteq K$ and so, $p \in Cn(K)$. Thus by inclusion, $p \in K^\ast$, that is contradictory. Assume the latter case, i.e., $P_I \in K'$. We have $(K' \setminus \{P_I\}) \cup \{ P_I\} \vdash p$. By deduction theorem we deduce that $K' \setminus \{ P_I\} \vdash P_I \to p$. Since $K' \setminus \{ P_I\} \subseteq K$, then $P_I \to p \in Cn(K)$. From inclusion, it follows that $P_I \to p \in K^\ast$. Since by assumption $K^\ast$ contains $P_I$, then by modus ponens and closure, it is easy to show that $p \in K^\ast$. That is again in contradiction with the assumption. For the case that not $b<R(I)$, by the definition the result is $K$ and by inclusion $K \subseteq K^\ast$. 
\end{proof}

\subsection{Contraction}

For contracting a given belief set $K$ with respect to an input $I$, we are aiming to eliminate $P_I$ from $K$ by the claim of $S_I$. We expect this change to happen whenever $I$ has more amount of reliability than the amount of believability of $P_I$ in $K$. So, we want to constraint $E(P_I)<R(I)$ be satisfied for contraction. In \cite{GM88}, by having a standard relation for epistemic entrenchment, e.g., $\leq_{EE}$, a construction for AGM contraction called entrenchment-based contraction is defined as follow:

\begin{equation} \label{EECD}
K-p=
\begin{cases}
\{ q \in K:p <_{EE} p \lor q \}&\textit{if } p \notin Cn(\varnothing)\\
K&\textit{otherwise}
\end{cases}
\end{equation}

In the above definition, $<_{EE}$ is defined as usual. Since every epistemic function induces a standard epistemic entrenchment relation on propositions, we can define source-sensitive contraction based on equation \ref{EECD} by changing its constraint.

\begin{definition}
For a belief set $K$ and an input $I$, the function $\dot{-}$ is a source-sensitive expansion if and only if:
\begin{center}
$K \dot{-} I=
\begin{cases}
\{p \in K: E(P_I) < E(P_I \lor p) \}&\textit{if } E(P_I)<R(I)\\
K&\textit{otherwise}
\end{cases}$
\end{center}
\end{definition}

Again, we should change some notations to have the correspondences of AGM contraction postulates. The important point is that we are compelled to consider postulates ($\dot{-}$5), ($\dot{-}$7) and ($\dot{-}$8) in single-source conditions, because in our multi-source framework there is not (and must not be) any relation between inputs with different sources. These postulates only make sense in single-source environments, like AGM. Similar to source-sensitive expansion, the success postulate\footnote{With respect the success postulate for contraction, if $-$ is an AGM contraction, then $p \notin K-p$ whenever $p \notin Cn(\varnothing)$.} will be replaced by a weakening of that which is called relative success. It specifies ``all or nothing'' approach for constructing our non-prioritized contraction. The corresponding postulates to AGM contraction postulates are:

\begin{description}
\item [($\dot{-}$1)]
$K \dot{-} I$ is a belief set. (closure)

\item [($\dot{-}$2)]
$P_I \notin K \dot{-} I$ or $K \dot{-} I=K$. (relative success)$^{\ast}$

\item [($\dot{-}$3)]
$K \dot{-} I \subseteq K$. (inclusion)

\item [($\dot{-}$4)]
If $P_I \notin K$, then $K \dot{-} I=K$. (vacuity)

\item [($\dot{-}$5)]
If $S_{I_1} = S_{I_2}$ and $Cn(P_{I_1})=Cn(P_{I_2})$, then $K \dot{-} I_1=K \dot{-} I_2$. (extensionality)

\item [($\dot{-}$6)]
$K \subseteq Cn((K \dot{-} I) \cup \{P_I\})$. (recovery)

\item [($\dot{-}$7)]
If $S_{I_1}=S_{I_2}=S_{I_3}$ and $P_{I_3} = P_{I_1} \land P_{I_2}$, then $K \dot{-} I_1 \cap K \dot{-} I_2 \subseteq K \dot{-} I_3$. (conjunctive overlap)

\item [($\dot{-}$8)]
If $S_{I_1}=S_{I_2}=S_{I_3}$ and $P_{I_3} = P_{I_1} \land P_{I_2}$ and $P_{I_1} \notin K \dot{-} I_3$, then $K \dot{-} I_3 \subseteq K \dot{-} I_1$. (conjunctive inclusion)
\end{description}

Now we can show the following theorem:

\begin{theorem}
Source-sensitive contraction satisfies ($-$1)-($-$8),
\end{theorem}

\begin{proof}
Again, we show the satisfiability of postulates for both cases $E(P_I)<R(I)$ and not $E(P_I)<R(I)$. In the first case, the properties of PAC provide the possibility of getting help of the proof of original case (\cite{GM88}). However, using PAC and being non-prioritized make some changes in proofs.

For ($-$1), suppose that $E(P_I)<R(I)$. We must show that by having a set $A \subseteq K \dot{-} I$ such that $p \in Cn(A)$, we can conclude that $p \in K \dot{-} I$. Clearly, each element of $A$ is also an element of $K$. Since $K$ is a belief set, $p$ is an element of $K$. By ($Cn$4), there is a finite set $A'=\{ p_1,p_2,...,p_n:p_i \in A \}$ such that $p \in Cn(A')$, i.e., $A' \vdash p$. By the rules of PAC, we have $p_\land := p_1 \land p_2 \land ... \land p_n \vdash p$. 

For showing that $p \in K \dot{-} I$, we must prove that $E(P_I) < E(P_I \lor p)$. Since each $p_i \in A'$ is an element of $K \dot{-} I$, then by definition for every $p_i$, $E(P_I)<E(P_I \lor p_i)$. By ($E$1) and ($E$2) it is easy to show that $E(P_I)<E((P_I \lor p_1) \land (P_I \lor p_2) \land ... \land (P_I \lor p_n))$. In PAC, $\vdash ((P_I \lor p_1) \land (P_I \lor p_2) \land ... \land (P_I \lor p_n)) \leftrightarrow (P_I \lor p_\land)$. So by ($E$1), $E(P_I)<E(P_I \lor p_\land)$. Since $p_\land \vdash p$, by the rules of PAC, $P_I \lor p_\land \vdash P_I \lor p$ and by ($E$1), $E(P_I \lor p_\land) \leq E(P_I \lor p)$. Therefore, it deduces that $E(P_I)<E(P_I \lor p)$. We are done.

For the case that not $E(P_I)<R(I)$, the result is $K$ that is a belief set.

For ($-$2), suppose that $E(P_I)<R(I)$. Clearly, not $E(P_I) < E(P_I \lor P_I)$. So by definition, $P_I \notin K \dot{-} I$. If it is not the case that $E(P_I)<R(I)$, then we have $K \dot{-} I=K$.

For ($-$3), since by definition, the elements of the output are also the elements of $K$. Hence, $K \dot{-} I \subseteq K$.

For ($-$4), assume that $P_I \notin K$. By ($E$3), we deduce that $E(P_I)=b$. Now if $E(P_I)<R(I)$, then the result is $\{p \in K: E(P_I)<E(P_I \lor p) \}$. By ($E$1) and ($E$3), it is easy to see that for every $p \in K$, $E(P_I)<E(p) \leq E(P_I \lor p)$. So $K \dot{-} I=K$. If not $E(P_I)<R(I)$, the result is obvious.

For ($-$5), since $Cn(P_{I_1})=Cn(P_{I_2})$, from reflexivity of $\vdash$ and deduction theorem, it follows that $\vdash P_{I_1} \leftrightarrow P_{I_2}$. Hence, by ($E$1) and ($R$) we conclude that $E(P_{I_1})=E(P_{I_2})$ and $R(I_1)=R(I_2)$. So, $E(P_{I_1})<R(I_1)$ if and only if $E(P_{I_2})<R(I_2)$. Assume that $E(P_{I_1})<R(I_1)$. So, $K \dot{-} I_1= \{p \in K: E(P_{I_1}) < E(P_{I_1} \lor p) \}$ and $K \dot{-} I_2= \{p \in K: E(P_{I_2}) < E(P_{I_2} \lor p) \}$. Suppose that $p \in K \dot{-} I_1$. It means that $p \in K$ and $(P_{I_1}) < E(P_{I_1} \lor p) \}$. Since $\vdash P_{I_1} \leftrightarrow P_{I_2}$, then $\vdash P_{I_1} \lor p \leftrightarrow P_{I_2} \lor p$. Hence by ($E$1), $(P_{I_2}) < E(P_{I_2} \lor p) \}$. So, $p \in K \dot{-} I_2$ and $K \dot{-} I_1 \subseteq K \dot{-} I_2$. The reverse is similar. Therefore $K \dot{-} I_1=K \dot{-} I_2$. Now assume not $E(P_{I_1}) < R(I_1)$. It deduce that not $E(P_{I_2}) < R(I_2)$. In this case, by definition we have $K \dot{-} I_1=K \dot{-} I_2=K$.

For ($-$6), first assume that $E(P_I)<R(I)$. Clearly by ($E$4) and the definition of $R$, $\nvdash P_I$. By definition, $Cn((K \dot{-} I) \cup \{P_I\})=Cn(\{ p \in K: E(P_I) \leq E(P_I \lor p)\} \cup \{P_I \})$. Now we want to show that $K \subseteq Cn((K \dot{-} I) \cup \{P_I\})$. Thus by taking $p \in K$, we will show that $p \in Cn((K \dot{-} I) \cup \{P_I\})$. Since in PAC for every $p$ and $q$ we have $p \vdash q \to p$, by closure of $K$, we conclude that $P_I \to p \in K$. Since $P_I \lor (P_I \to p)$ is theorem in PAC and $\nvdash P_I$, then by ($E$4), $E(P_I) < E(P_I \lor (P_I \to p))$. Hence by definition, $P_I \to p \in K \dot{-} I$ and by ($Cn$3), $P_I \to p \in Cn((K \dot{-} I) \cup \{P_I\})$. From ($Cn$1), it follows that $P_I \in Cn((K \dot{-} I) \cup \{P_I\})$. So by modus ponens, it deduces that $p \in Cn((K \dot{-} I) \cup \{P_I\})$. Hence, $K \subseteq Cn((K \dot{-} I) \cup \{P_I\})$. Now suppose that not $E(P_I)<R(I)$. Then by definition, $Cn((K \dot{-} I) \cup \{P_I\}) =Cn(K \cup \{P_I \})$ and the result is obvious.

For ($-$7), assume that $E(P_{I_3})<R(I_3)$. Independently, three conditions are possible for $E(P_{I_1})$ and $E(P_{I_2})$. $E(P_{I_1})<E(P_{I_2})$ or $E(P_{I_2})<E(P_{I_1})$ or $E(P_{I_1})=E(P_{I_2})$. Consider the first case. By ($E$1) and ($E$2), we have $E(P_{I_1}) \leq E(P_{I_3})$. Now, ($R$) deduces that $R(I_3) \leq R(I_1)$. So, we have $E(P_{I_1})<R(I_1)$. We want to show that $K \dot{-} I_1 \subseteq K \dot{-} I_3$. Suppose that $p \in K \dot{-} I_1$. It means that $p \in K$ and $E(P_{I_1})<E(P_{I_1} \lor p)$. It deduces that $E(P_{I_3})<E(P_{I_1} \lor p)$. Now since $E(P_{I_1})<E(P_{I_2})$, by ($E$1) and ($R$), we have $E(P_{I_3}) \leq E(P_{I_1}) <E(P_{I_2}) \leq E(P_{I_2} \lor p)$. Hence by ($E$1) and ($E$2), $E(P_{I_3})<E((P_{I_1} \lor p) \land (P_{I_2} \lor p))$. Since $\vdash (P_{I_3} \lor p) \leftrightarrow (P_{I_1} \lor p) \land (P_{I_1} \lor p)$, from ($E$1) we conclude $E(P_{I_3})<E(P_{I_3} \lor p)$. So by definition, $p \in K \dot{-} I_3$ and $K \dot{-} I_1 \subseteq K \dot{-} I_3$. Therefore, $K \dot{-} I_1 \cap K \dot{-} I_2 \subseteq K \dot{-} I_3$. For the case that $E(P_{I_2})<E(P_{I_1})$, the proof is similar.

Now assume the third case, i.e., $E(P_{I_1})=E(P_{I_2})$. So, since $E(P_{I_3})<R(I_3)$, both constraints $E(P_{I_1})<R(I_1)$ and $E(P_{I_2})<R(I_2)$ are satisfied. Now take $p \in K \dot{-} I_1 \cap K \dot{-} I_2$. It means that $p \in K$ and $E(P_{I_3}) \leq E(P_{I_1})<E(P_{I_1} \lor p)$ and $E(P_{I_3}) \leq E(P_{I_2})<E(P_{I_2} \lor p)$. Therefore $E(P_{I_3})<E((P_{I_1} \lor p) \land (P_{I_1} \lor p))$. So as we seen, it deduces that $E(P_{I_3})<E(P_{I_3} \lor p)$. Hence by definition, $p \in K \dot{-} I_3$ and $K \dot{-} I_1 \cap K \dot{-} I_2 \subseteq K \dot{-} I_3$. The only remained case is when not $E(P_{I_3})<R(I_3)$. In this case, by definition $K \dot{-} I_3=K$. From ($\dot{-}$2) it follows that $K \dot{-} I_1 \subseteq K$ and $K \dot{-} I_2 \subseteq K$. Therefore again, $K \dot{-} I_1 \cap K \dot{-} I_2 \subseteq K \dot{-} I_3$.

For ($-$8), first suppose that $E(P_{I_1})<R(I_1)$. It follows from assumption $P_{I_1} \notin K \dot{-} I_3$, that either $P_{I_1} \notin K$ or it is the case that $P_{I_1} \in K$ and $E(P_{I_3})<R(I_3)$ and not $E(P_{I_3})<E(P_{I_3} \lor P_{I_1})$. In the first case, by ($\dot{-}$3), $K \dot{-} I_1=K$ and the postulate is satisfied. Now take the second case. We can show that if not $E(P_{I_3})<E(P_{I_3} \lor P_{I_1})$, then we have $E(P_{I_1}) \leq E(P_{I_2})$. It is because if $E(P_{I_2})<E(P_{I_1})$, then by ($E$1) we conclude that $E(P_{I_2})<E(P_{I_1} \lor P_{I_2})$. Then by having ($E$2), we have $E(P_{I_2})<E((P_{I_1} \lor P_{I_2}) \land P_{I_1})$. Since in PAC $\vdash ((P_{I_1} \lor P_{I_2}) \land P_{I_1}) \leftrightarrow (P_{I_3} \lor P_{I_1})$, then ($E$1), deduces $E(P_{I_3}) \leq E(P_{I_2})<E((P_{I_3} \lor P_{I_1}))$. Contradiction. So, $E(P_{I_1}) \leq E(P_{I_2})$.

Now, we must show that in this condition, we have $K \dot{-} I_3 \subseteq K \dot{-} I_1$. Take $p \notin K \dot{-} I_1$. It means either $p \notin K$ or $p \in K$ and $E(P_{I_1})<R(I_1)$ and not $E(P_{I_1})<E(P_{I_1} \lor p)$. The first case clearly deduces that $p \notin K \dot{-} I_3$. Assume the second case. Since we have $E(P_{I_1}) \leq E(P_{I_2})$ and $\vdash (P_{I_3} \lor p) \leftrightarrow ((P_{I_1} \lor p) \land (P_{I_2} \lor p))$, then by ($E$1) and ($E$2) and assumption not $E(P_{I_1})<E(P_{I_1} \lor p)$, we conclude that $E(P_{I_3} \lor p) \leq E(P_{I_1} \lor p) \leq E(P_{I_1}) \leq E(P_{I_3})$. We had before $E(P_{I_3})<R(I_3)$. Hence, by definition $p \notin K \dot{-} I_3$. The only remaining case is when not $E(P_{I_1})<R(I_1)$. In this case, by definition we have $K \dot{-} I_1 =K$, Therefore by ($\dot{-}$2), $K \dot{-} I_3 \subseteq K \dot{-} I_1$.
\end{proof}

\subsection{Defining Revision}

As said before, revision should not be considered as a main kind of belief change in the paraconsistent context. However, it is possible to define it as a derived change from expansion and contraction. The standard way of doing this, is Levi identity, i.e.:

\begin{equation}
K \ast_{L} p=Cn((K- \neg p) \cup \{p\})
\end{equation}

It is shown that if $-$ is an AGM contraction function, then $\ast_L$ will be an AGM revision. The original form of Levi identity is not suitable in our framework, for two reasons. First, in our multi-source environment there is no relation between an input with proposition $p$ and another one with proposition $\neg p$. Second, since the last step of Levi identity is expansion, it almost disregards our non-prioritized approach.

Thus, it should be modified in order to fit our proposed framework. For every input $I=(p,s)$, we use the auxiliary notation $\overline{I}$, such that $\overline{I}=(\neg p,s)$. Now, the revision process on $K$ with respect to $I$ will be contracting it by $\overline{I}$ and then expanding it by $I$. These changes will be done whenever their considered restrictions are satisfied. Therefore, one way to define the source-sensitive revision function, denoted by $\dot{\ast}$, is as follows:

\begin{equation}
K \dot{\ast} I=
\begin{cases}
(K \dot{-} \overline{I}) \dot{+} I & \textit{if } E(P_{\overline{I}})<R(\overline{I}) \textit{ and } b<R(I)\\
K & \textit{otherwise}
\end{cases}
\end{equation}

However, it is not the only way. Another approach can be modifying reverse Levi identity (\cite{Han93}), that is:

\begin{equation}
K \ast_{RL} p = (K+p)- \neg p
\end{equation}

In the original case, whenever $\neg p \in K$ we will have $K+p=K_\bot$. Hence, this definition does not fit the AGM model well. One way to avoid triviality, is to use belief bases (\cite{Fur91, Han91}) instead of belief sets. Belief bases are sets of propositions that are not closed under logical consequences. Since in our framework inconsistencies do not explode to triviality, the problem will not happen. Thus by some necessary modifications, also reverse Levi identity can be used for defining source-sensitive revision.

However, by avoiding primacy of the input and consistency criterion, many of AGM revision postulates will be lost. It is important for many various applications of belief revision, to find the best way of this kind of change.

\section{Conclusion}

In this paper, we introduced a new framework, called source-sensitive belief change, based on AGM to address three criticisms applied to the AGM model, i.e., single-agent environment, primacy of input and consistency criterion. We specified our position in each extended subject for those criticisms, called multi-agent belief change, non-prioritized belief change and paraconsistent belief change. Naturally, there are some similarities and relations between our model and other proposals (\cite{DG96}, \cite{RS95}, \cite{Tam12} and \cite{Tam14}) in these fields and also there are several motivations, features and properties that are exclusive in the recent framework. 

As showed, the changes we applied for constructing our desired model, resulted in persevering AGM postulates as much as possible. However, some definitions we used are new and defined with very general attributes. Hence, ways of improving them are open and need to be studied.

\end{document}